\documentclass[final]{cvpr}
\pdfoutput=1

\usepackage{times}
\usepackage{epsfig}
\usepackage{graphicx}
\usepackage{amsmath}
\usepackage{amssymb,amsfonts,bbm}
\usepackage{enumitem}
\usepackage{amsthm}
\usepackage{optidef}
\usepackage{microtype}
\usepackage{booktabs}
\usepackage{multirow}
\usepackage{pifont}
\newcommand{\cmark}{\ding{51}}%
\newcommand{\xmark}{\ding{55}}%

\usepackage[ruled,vlined]{algorithm2e}
\usepackage{float}
\newtheorem{lemma}{Lemma}
\usepackage[dvipsnames]{xcolor}
\usepackage{array}
\usepackage{nopageno}
\newcolumntype{L}{>{\centering\arraybackslash}m{1.3cm}}

\usepackage[pagebackref=true,breaklinks=true,colorlinks,bookmarks=false]{hyperref}

\def\cvprPaperID{11153} 
\def\confYear{CVPR 2021}

\begin{document}
\pagestyle{plain}
\title{Unsupervised Multi-source Domain Adaptation Without Access to Source Data}

\author{Sk Miraj Ahmed$^{1,}$\thanks{Equal Contribution},  Dripta S. Raychaudhuri$^{1,\ast}$, Sujoy Paul$^{2,\ast,}$\thanks{Work done while SP was a PhD student at UC Riverside.}, Samet Oymak$^{1}$, Amit K. Roy-Chowdhury$^{1}$\\
$^{1}$ University of California, Riverside,  $^{2}$ Google Research\\
{\tt \small \{sahme047@, drayc001@, spaul003@, oymak@ece.,  amitrc@ece.\}ucr.edu}
}

\maketitle

\begin{abstract}
Unsupervised Domain Adaptation (UDA) aims to learn a predictor model for an unlabeled domain by transferring knowledge from a separate labeled source domain. However, most of these conventional UDA approaches make the strong assumption of having access to the source data during training, which may not be very practical due to privacy, security and storage concerns. A recent line of work addressed this problem and proposed an algorithm that transfers knowledge to the unlabeled target domain from a single source model without requiring access to the source data. However, for adaptation purposes, if there are multiple trained source models available to choose from, this method has to go through adapting each and every model individually, to check for the best source. Thus, we ask the question: can we find the optimal combination of source models, with no source data and without target labels, whose performance is no worse than the single best source?
To answer this, we propose a novel and efficient algorithm which automatically combines the source models with suitable weights in such a way that it performs at least as good as the best source model. We provide intuitive theoretical insights to justify our claim. Furthermore, extensive experiments are conducted on several benchmark datasets to show the effectiveness of our algorithm, where in most cases, our method not only reaches best source accuracy but also outperforms it. 

\end{abstract}

\section{Introduction}

Deep neural networks have achieved proficiency in a multiple array of vision tasks \cite{he2016deep,long2015fully,kirillov2019panoptic,redmon2016you}, however, these models have consistently fallen short in adapting to visual distributional shifts \cite{luo2019taking}. Human recognition, on the other hand, is robust to such shifts, such as reading text in a new font or recognizing objects in unseen environments. Imparting such robustness towards distributional shifts to deep models is fundamental in applying these models to practical scenarios. 

Unsupervised domain adaptation (UDA) \cite{ben2010theory,saenko2010adapting} seeks to bridge this performance gap due to domain shift via adaptation of the model on small amounts of unsupervised data from the target domain. The majority of current approaches \cite{ganin2016domain,hoffman2018cycada} optimize a two-fold objective: (i) minimize the empirical risk on the source data, (ii) make the target and source features indistinguishable from each other. Minimizing distribution divergence between domains by matching the distribution statistical moments at different orders have also been explored extensively \cite{sun2016return,peng2019moment}.

\begin{figure}[t]
\includegraphics[width=0.48\textwidth]{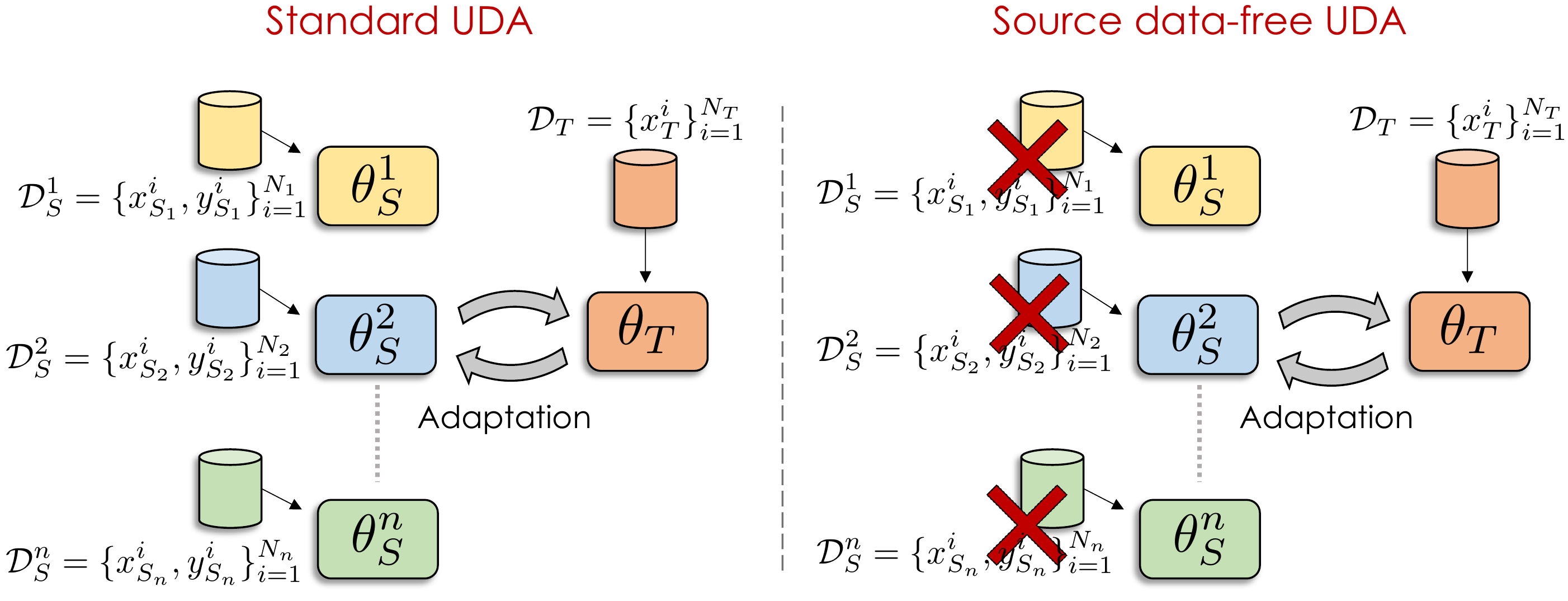}
\vskip 0.1in
\caption{\textbf{Problem setup.} Standard unsupervised multi-source domain adaptation (UDA) utilizes the source data, along with the models trained on the source, to perform adaptation on a target domain. In contrast, we introduce a setting which adapts multiple models without requiring access to the source data.} 
\label{fig:intro}
\end{figure}

A shortcoming of all the above approaches is the transductive scenario in which they operate, i.e., the source data is required for adaptation purposes. In a real-world setting, source data may not be available for a variety of reasons. Privacy and security are the primary concern, with the data possibly containing sensitive information. Another crucial reason is storage issues, i.e., source datasets may contain videos or high-resolution images and it might not be practical to transfer or store on different platforms. Consequently, it is imperative to develop unsupervised adaptation approaches which can adapt the source models to the target domain \textit{without access to the source data}. 

Recent works \cite{li2020model,liang2020we} attempt this by adapting a single source model to a target domain without accessing the source data. However, an underlying assumption of these methods is that the most correlated source model is provided by an oracle for adaptation purposes. A more challenging and practical scenario entails adaptation from a \textit{bag of source models} - each of these source domains are correlated to the target by different amounts and adaptation involves not only incorporating the combined prior knowledge from multiple models, but simultaneously preventing the possibility of negative transfer. In this paper, we introduce the problem of unsupervised \textit{multi-source adaptation without access to source data}. We develop an algorithm based on the principles of pseudo-labeling and information maximization and provide intuitive theoretical insights to show that our framework guarantees performance better than the best available source and minimize the effect of negative transfer. 

To solve this problem of multiple source model adaptation without accessing the source data, we deploy \textit{Information Maximization (IM)} loss \cite{liang2020we} on the weighted combination of target soft labels from all the source models. 
We also use the pseudo-label strategy inspired from deep cluster method \cite{caron2018deep}, along with the IM loss to minimize noisy cluster assignment of the features. The overall optimization jointly adapts the feature encoders from sources as well as the corresponding source weights, combining which the target model is obtained. 
\noindent {\bf Main Contributions.} We address the problem of multiple source UDA, with no access to the source data. Towards solving the problem, we make the following contributions: \\
    $\bullet$ We propose a novel UDA algorithm which operates without requiring access to the source data. We term it as Data frEe multi-sourCe unsupervISed domain adaptatiON (DECISION). Our algorithm automatically identifies the optimal blend of source models to generate the target model by optimizing a carefully designed unsupervised loss.\\
    $\bullet$ Under intuitive assumptions, we establish theoretical guarantees on the performance of the target model which shows that it is consistently at least as good as deploying the single best source model, thus, minimizing negative transfer.\\
    $\bullet$ We validate our claim by extensive numerical experiments, demonstrating the practical benefits of our approach.



\section{Related works}
In this section we present a brief overview of the literature in the area of unsupervised domain adaptation in both the single and multiple sources scenario, as well as the closely related setting of hypothesis transfer learning.\\
\noindent \textbf{Unsupervised domain adaptation.} UDA methods have been used for a variety of tasks, including image classification \cite{tzeng2017adversarial}, semantic segmentation \cite{paul2020domain} and object detection \cite{hsu2020progressive}. Besides the feature space adaptation methods based on the paradigms of moment matching \cite{sun2016return,peng2019moment} and adversarial learning \cite{ganin2016domain, tzeng2017adversarial}, recent works have explored pixel space adaptation via image translation \cite{hoffman2018cycada}. All existing UDA methods require access to labeled source data, which may not be available in many applications. \\
\noindent \textbf{Hypothesis transfer learning.} Similar to our objective, hypothesis transfer learning (HTL) \cite{singh2018nonparametric,perrot2015theoretical, ahmed2020camera} aims to transfer learnt source hypotheses to a target domain without access to source data. However, data is assumed to be labeled in the target domain in contrast to our scenario, limiting its applicability to real-world settings. Recently, \cite{li2020model,liang2020we} extend the standard HTL setting to unsupervised target data (U-HTL) by adapting single source hypotheses via pseudo-labeling. Our paper takes this one step further by introducing multiple source models, which may or may not be positively correlated with the target domain.\\ 
\noindent \textbf{Multi-source domain adaptation.} Multi-source domain adaptation (MSDA) extends the standard UDA setting by incorporating knowledge from multiple source models. Latent space transformation methods \cite{zhao2020multi} aim to align the features of different domains by optimizing a discrepancy measure or an adversarial loss. Discrepancy based methods seek to align the domains by minimizing measures such as maximum mean discrepancy \cite{guo2018multi,zhao2020multi} and R\'enyi-divergence \cite{hoffman2018algorithms}. Adversarial methods aim to make features from multiple domains indistinguishable to a domain discriminator by optimizing GAN loss \cite{xu2018deep}, $\mathcal{H}-$divergence \cite{zhao2018adversarial} and Wasserstein distance \cite{wang2019tmda,li2018extracting}. Domain generative methods \cite{russo2019towards,lin2020multi} use some form of domain translation, such as the CycleGAN \cite{zhu2017unpaired}, to perform adaptation at the pixel level. All these methods assume access to the source data during adaptation. 

\section{Methodology}
\begin{table}
\small
\centering
\resizebox{0.47\textwidth}{!}{
\begin{tabular}{@{}cLLLL@{}}
\toprule
\textsc{Method} & \textsc{Multiple domains} & \textsc{No source data} & \textsc{Source model} & \textsc{Unlabeled target data} \\ \midrule
UDA \cite{hoffman2018cycada}   
& \color{red}\xmark
& \color{red}\xmark    
& \color{ForestGreen}\cmark 
& \color{ForestGreen}\cmark \\
MSDA \cite{peng2019moment}
& \color{ForestGreen}\cmark
& \color{red}\xmark
& \color{ForestGreen}\cmark 
& \color{ForestGreen}\cmark \\
HTL \cite{singh2018nonparametric}
& \color{red}\xmark
& \color{ForestGreen}\cmark
& \color{ForestGreen}\cmark
& \color{red}\xmark \\
U-HTL \cite{liang2020we}
& \color{red}\xmark
& \color{ForestGreen}\cmark
& \color{ForestGreen}\cmark
& \color{ForestGreen}\cmark \\
DECISION(Ours)   
& \color{ForestGreen}\cmark
& \color{ForestGreen}\cmark
& \color{ForestGreen}\cmark
& \color{ForestGreen}\cmark \\
\bottomrule \\
\end{tabular}
}
\caption{\textbf{Comparison to different adaptation settings by attributes demonstrated in the paper.} Our proposed setting satisfies all the criteria desired in a holistic adaptation framework.}
\vskip -0.1in
\end{table}
\noindent\textbf{Problem setting.}
We address the problem of jointly adapting multiple models, trained on a variety of domains, to a new target domain with access to only samples without annotations from the target. In this work, we will be considering the adaptation of classification models with $K$ categories and the input space being $\mathcal{X}$. Formally, let us consider we have a set of source models $\{\theta_S^j\}_{j=1}^n$, where the $j^{th}$ model $\theta_S^j: \mathcal{X} \rightarrow \mathbb{R}^K$, is a classification model learned using the source dataset $\mathcal{D}_S^j=\{x_{S_j}^i, y_{S_j}^i\}_{i=1}^{N_j}$, with $N_j$ data points, where $x_{S_j}^i$ and $y_{S_j}^i$ denote the $i$-th source image and the corresponding label respectively. Now, given a target unlabeled dataset $\mathcal{D}_T=\{x^i_T\}_{i=1}^{N_T}$, the problem is to learn a classification model $\theta_T: \mathcal{X} \rightarrow \mathbb{R}^K$, using only the learned source models, without any access to the source datasets. Note that this is different from multi-source domain adaptation methods in literature, which also utilize the source data while learning the target model $\theta_T$.

\begin{figure*}[t]
\centering
\includegraphics[width=\textwidth,height=0.45\textwidth]{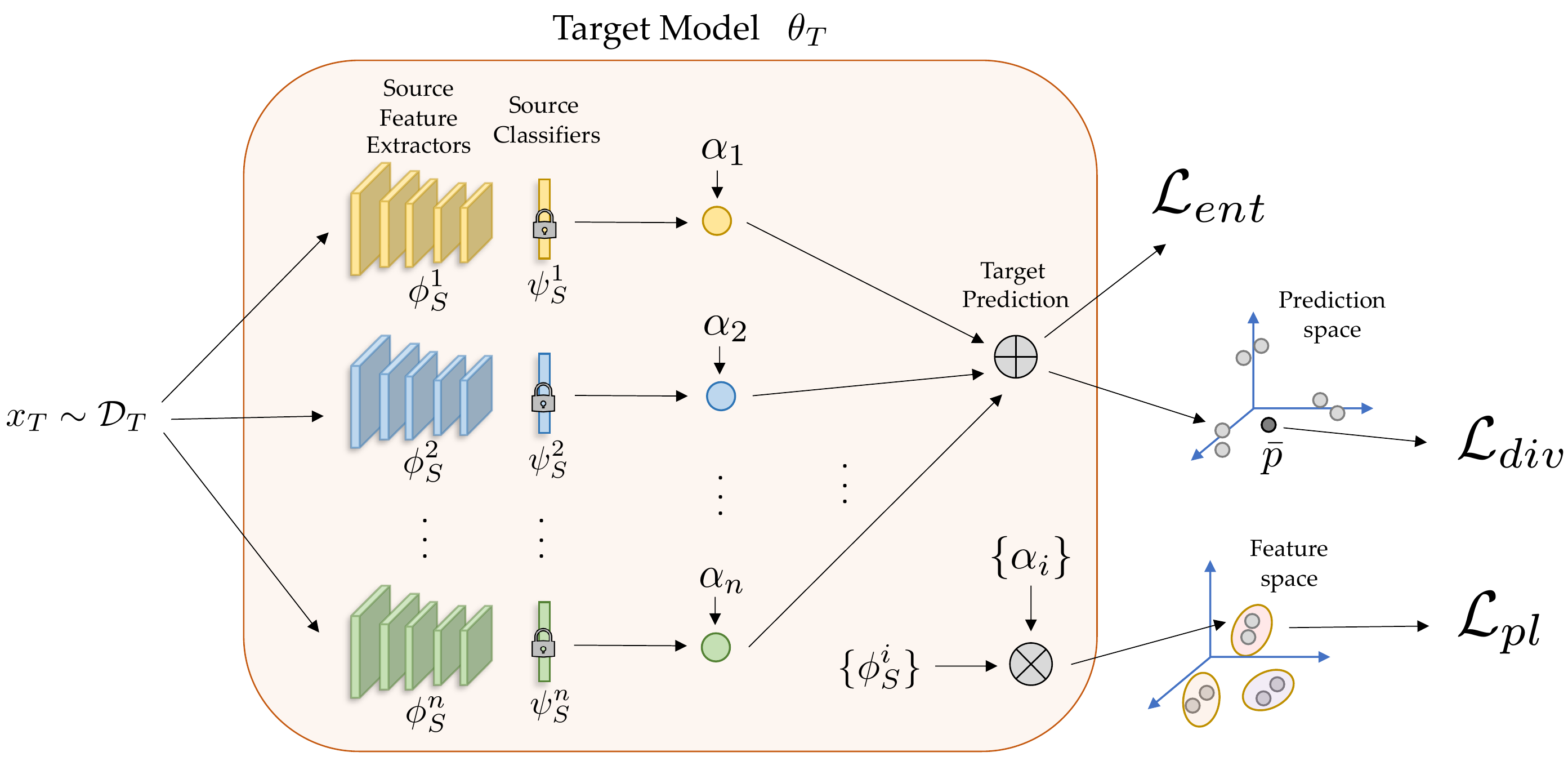} 
\vspace{0.1mm}
\caption{\textbf{Overall framework of our approach:} We freeze the final classification layers of all the sources and jointly optimize for the source feature encoders along with it's corresponding weights to get the target predictor by combining those. }
\vskip -0.1in
\label{framework}
\end{figure*}

\noindent
\textbf{Overall Framework.} We can decompose each of the source models into two modules: the feature extractor $\phi_S^i:\mathcal{X} \rightarrow \mathbb{R}^{d_i}$ and the classifier $\psi^i_S: \mathbb{R}^{d_i} \rightarrow \mathbb{R}^K$. Here, $d_i$ refers to the feature dimension of the $i$-th model while $K$ refers to the number of categories.
We aim to estimate the target model $\theta_T$ by combining knowledge only from the given source models in a manner that automatically rejects poor source models, i.e., those which are irrelevant for the target domain.


At the core of our framework lies a model aggregation scheme \cite{mansour2009domain,hoffman2018algorithms}, wherein we learn a set of weights $\{\alpha_i\}_{i=1}^n$ corresponding to each of the source models,
such that, $\alpha_k \geq 0$ and $\sum_{k=1}^n \alpha_k=1$. These weights 
represent a probability mass function over the source domains, with a higher value implying higher transferability from that particular domain, and are used to combine the source hypotheses accordingly. However, unlike previous works, we jointly adapt each individual model and simultaneously learn these weights by utilizing solely the unlabeled target instances. In what follows, we describe our training strategy used to achieve this in detail.

\subsection{Weighted Information Maximization} As we do not have access to the labeled source or target data, we propose to fix the source classifiers,$\{\psi^i_S\}_{i=1}^n$, since it contains the class distribution information of the source domain  and adapt solely the feature maps $\{\phi_S^i\}_{i=1}^n$ via the principle of information maximization \cite{bridle1992unsupervised,krause2010discriminative,oymak2020statistical,liang2020we}. Our motivation behind the adaptation process stems from the \textit{cluster assumption} \cite{chapelle2009semi} in semi-supervised learning, which hypothesizes that the discriminative model’s decision boundaries should be located in regions of the input space which are not densely populated. To achieve this, we minimize a conditional entropy term (i.e., for a given input example) \cite{grandvalet2005semi} as follows:
\begin{equation}
    \mathcal{L}_{\text{ent}} = -\mathbb{E}_{x_T \in \mathcal{D}_T}\Big[ \sum_{j=1}^K  \delta_j(\theta_T(x_T)) \log(\delta_j(\theta_T(x_T)))\Big]
\label{Im_ent}
\end{equation}
where $\theta_T(x_T)=\sum_{j=1}^n \alpha_j \theta_S^j(x_T)$, and $\delta(\cdot)$ denotes the softmax operation with $\delta_j(v)=\frac{\exp(v_j)}{\sum_{i=1}^K \exp(v_i)}$ for $v \in \mathbb{R}^K$. Intuitively, if a source $\theta_S^j$ has good transferability on the target and consequently, has smaller value of the conditional entropy, optimizing the term \eqref{Im_ent} over $\left\{\theta_S^j,\alpha_j\right\}$, will result in higher value of $\alpha_j$ than rest of the weights.

While entropy minimization effectively captures the cluster assumption when training with partial labels, in an unsupervised setting, it may lead to degenerate solutions, such as, always predicting a single class in an attempt to minimize conditional entropy. To control such degenerate solutions, we incorporate the idea of class diversity: configurations in which class labels are assigned evenly across the dataset are preferred. A simple way to encode our preference
towards class balance is to maximize the entropy of the empirical label distribution \cite{bridle1992unsupervised} as follows,
\begin{equation}
\label{Im_div}
    \mathcal{L}_{\text{div}} =  \sum_{j=1}^K -\Bar{p}_j\log\Bar{p}_j
\end{equation}
where $\Bar{p}= \mathbb{E}_{x_T \in \mathcal{D}_T}[\delta(\theta_T(x_T))]$. Combining the terms \eqref{Im_ent} and \eqref{Im_div}, we arrive at,
\begin{equation}
\label{Im_tot}
    \mathcal{L}_{\text{IM}} = \mathcal{L}_{\text{div}} -
    \mathcal{L}_{\text{ent}}
\end{equation}
which is the empirical estimate of the mutual information between the target data and the labels under the aggregate model $\theta_T$. Although maximizing this loss makes the predictions on the target data more confident and globally diverse, it may sometime still fail to restrict erroneous label assignment. Inspired by \cite{liang2020we}, we propose a pseudo-labeling strategy in an effort to contain this mislabeling.

\subsection{Weighted Pseudo-labeling}
As a result of domain shift, information maximization may result in some instances being clubbed with the wrong class cluster. These wrong predictions get reinforced over the course of training and lead to a phenomenon termed as \textit{confirmation bias} \cite{tarvainen2017mean}. Aiming to contain this effect we adopt a self-supervised clustering strategy \cite{liang2020we} inspired from the DeepCluster technique \cite{caron2018deep}. 

First, we calculate the cluster centroids induced by each source model for the whole target dataset as follows,
\begin{equation}
  \mu_{k_j}^{(0)} = \frac{\sum_{x_T\in \mathcal{D}_T}\delta_k(\hat{\theta}_S^j(x_T))\hat{\phi}_S^j(x_T)}{\sum_{x_T\in \mathcal{D}_T}\delta_k(\hat{\theta}_S^j(x_T))}
\end{equation}
where the cluster centroid of class $k$ obtained from source $j$ at iteration $i$ is denoted as $\mu_{k_j}^{(i)}$, and $\hat{\theta}_S^j=(\psi_S^j \circ \hat{\phi}_S^j)$ denotes the source from the previous iteration. These source-specific centroids are combined in accordance to the current aggregation weights on each source model as follows,
\begin{equation}
    \mu_k^{(0)} = \sum_{j=1}^n \alpha_j \mu_{k_j}^{(0)}
\end{equation}
Next, we compute the pseudo-label of each sample by assigning it to its nearest cluster centroid in the feature space,
\begin{equation}
    \hat{y}_{T}^{(0)} = \text{arg} \ \underset{k}{\text{min}} \ \|\hat{\theta}_T(x_T)-\mu_k^{(0)}\|_2^2
\label{pseudo}
\end{equation}
We reiterate this process to get the updated centroids and pseudo-labels as follows,
\begin{equation}
    \mu_{k_j}^{(1)} = \frac{\sum_{x_T\in \mathcal{D}_T} \mathbbm{1}\{\hat{y}_{T}^{(0)}=k\}\hat{\phi}_S^j(x_T)}{\sum_{x_T\in \mathcal{D}_T}\mathbbm{1}(\hat{y^{t_0}}=k)}
\end{equation}
\begin{equation}
    \mu_k^{(1)} = \sum_{j=1}^n \alpha_j \mu_{k_j}^{(1)}
\end{equation}
\begin{equation}
    \hat{y}_{T}^{(1)} = \text{arg} \ \underset{k}{\text{min}} \ \|\hat{\theta}_T(x_T)-\mu_k^{(1)}\|_2^2
\end{equation}
where $\mathbbm{1}(\cdot)$ is an indicator function which gives a value of $1$ when the argument is true. While this alternating process of computing cluster centroids and pseudo-labeling can be repeated multiple times to get stationary pseudo-labels, one round is sufficient for all practical purposes. 
We then obtain the cross-entropy loss w.r.t. these pseudo-labels as follows:
\begin{equation}
    \mathcal{L}_{\text{pl}}(Q_T,\theta_T) = - \mathbb{E}_{x_T\in \mathcal{D}_T}\sum_{k=1}^K \mathbbm{1}\{\hat{y}_T=k\} \log \delta_k (\theta_T(x_T)).
\label{Im_pl}
\end{equation}
Note that the pseudo-labels are updated regularly after a certain number of iterations as discussed in Section \ref{experiments}.

\subsection{Optimization}
In summary, given $n$ source hypothesis $\{\theta_S^j\}_{j=1}^n=\{\psi_S^j \circ \phi_S^j\}_{j=1}^n$ and target data $\mathcal{D}_T=\{x^i_T\}_{i=1}^{n_T}$, we fix the classifier from each of the sources and optimize over the parameters of $\{\phi_S^j\}_{j=1}^n$ and the aggregation weights $\{\alpha_j\}_{j=1}^n$. The final objective is given by,
\begin{equation}
\begin{split}
    \mathcal{L}_{tot}=\mathcal{L}_{\text{ent}} - \mathcal{L}_{\text{div}} + \lambda \mathcal{L}_{\text{pl}}
\end{split}
\label{overall_obj}
\end{equation}

The above objective is used to solve the following optimization problem,
\begin{mini}|l|
{\{\phi_S^j\}_{j=1}^n,\{\alpha_j\}_{j=1}^n}{\mathcal{L}_{tot}}{}{}
\addConstraint {\alpha_j \geq 0, \forall j \in \{1, 2, \dots, n\}}
\addConstraint {\sum_{j=1}^n \alpha_j=1}
\label{opt:main_opt}
\end{mini}

Once we obtain the optimal set of $\phi_S^{j*}$ and $\alpha_j^*$, the optimal target hypothesis is computed as
$\theta_T= \sum_{j=1}^n \alpha_j^* (\psi_S^j \circ \phi_S^{j*})$.
To solve the optimization \eqref{opt:main_opt} we follow the steps of Algorithm~\eqref{algo1} stated below.

\begin{algorithm}[]
\SetAlgoLined
\textbf{Input:} Trained source models $\{\theta_S^j\}_{j=1}^n=\{\psi_S^j\circ\phi_S^j\}_{j=1}^n$, unlabeled target data $\{x^i_T\}_{i=1}^{N_T}$,weight parameters $\{\alpha_j\}_{j=1}^n$,max number of epochs $E$, regularization parameter $\lambda$,number of batches $B$  \\
\textbf{Output:} Optimal feature enocoders $\{\phi_S^{j*}\}_{j=1}^n$, optimal source weights $\{\alpha_j^*\}_{j=1}^n$ \\
\textbf{Initialization}: Freeze final classification layers $\{\psi_S^j\}_{j=1}^n$, set $\alpha_j=1$ for all $j$\\
 \For{$epoch=1$ $\textbf{to}$ $E$}{
  Calculate pseudo-labels from equation~\eqref{pseudo} \\
  Calculate the mean embedding $\bar{p}$ from equation~\eqref{Im_div} \\
  \For{$iteration=1$ \textbf{to} $B$}{
  
  Sample a mini batch from target and pass it through each of the source models \\
  
  calculate all the losses from equation~\eqref{Im_ent},\eqref{Im_div} and ~\eqref{Im_pl} \\
  
  calculate total loss from equation~\eqref{Im_tot} \\
  
  Update the parameters in $\{\phi_S^{j}\}_{j=1}^n$ and $\{\alpha_j\}_{j=1}^n$ from optimization\eqref{opt:main_opt} \\
  
  Make $\alpha$ positive by setting $\alpha_j=1/(1+e^{-\alpha_j})$ \\
 
  Normalize $\alpha$ by setting $\alpha_j=\alpha_j/\sum_{i=1}^n \alpha_i$
  }
 
 }
 \caption{Algorithm to Solve Eq.~\ref{opt:main_opt}} 
 \label{algo1}
\end{algorithm}

\section{Theoretical Insights}
\noindent \textbf{Theoretical motivation behind our approach.}
Our algorithm aims to find the optimal weights $\{\alpha_j\}_{j=1}^n$ for each source and takes a convex combination of the source predictors to obtain the target predictor. Here, we shall show that under intuitive assumptions on the source and target distributions, there exists a simple choice of target predictor, which can perform better than or equal to the best source model being applied directly on the target data. 


Formally, let $L$ be a loss function which maps the pair of model-predicted label and the ground-truth label to a scalar. Denote the expected loss over $k$-th source distribution $Q^k_S$ using the source predictor $\theta$ via $\mathcal{L}(Q^k_S,\theta)=\mathbb{E}_x[L(\theta(x),y)]=\int_{x} L(\theta(x),y)Q^k_S(x)dx$.  Now let $\theta_S^k$ be the optimal source predictor given by $\theta^k_S = \text{arg} \ \underset{\theta}{\text{min}} \ {\cal{L}}(Q^k_S, \theta)\ \forall \ 1\leq k\leq n$.
Let us also assume that the target distribution is in the span of source distributions. We formalize this by expressing the target distribution as an affine combination of source distributions i.e., $Q_T(x)={\sum_{k=1}^n \lambda_k Q^k_S(x): \lambda_k \geq 0, \sum_{k=1}^n \lambda_k =1}$. Under this assumption, if we express our target predictor
as $\theta_T(x)=\sum_{k=1}^n \frac{\lambda_k Q^k_S(x)}{\sum_{j=1}^n \lambda_j Q^j_S(x)} \theta_{S}^k(x)$, then we establish our theoretical claim stated in Lemma~\ref{lemma1}.

\begin{lemma}
\label{lemma1}
Assume that the loss $L(\theta(x),y)$ 
is convex in its first argument and that there exists a $\lambda \in \mathbb{R}^n$ where $\lambda \geq 0$ and $\lambda^\top \mathbbm{1}=1$, such that the target distribution is exactly equal to the mixture of source distributions, i.e., $Q_T=\sum_{i=1}^n \lambda_i Q^i_S$. Set the target predictor as the following convex combination of the optimal source predictors
$$\theta_T(x)=\sum_{k=1}^n \frac{\lambda_k Q^k_S(x)}{\sum_{j=1}^n \lambda_j Q^j_S(x))} \theta_{S}^k(x).$$ 
Recall the pseudo-labeling loss \eqref{Im_pl}. Then, for this target predictor, over the target distribution, the unsupervised loss induced by the pseudo-labels and the supervised loss are both less than or equal to the loss induced by the best source predictor. In particular, 
\[
\mathcal{L}(Q_T,\theta_T) \leq \underset{1\leq j\leq n}{\min}\ \mathcal{L}(Q_T,\theta_{S}^j).
\]
Let $\alpha=\arg\min_{1\leq j\leq n}\ \mathcal{L}(Q_T,\theta_{S}^j)$. Additionally, this inequality is strict if the entries of $\lambda$ are strictly positive and there exists a source $i$ for which the strict inequality $\mathcal{L}(Q^i_S,\theta_{S}^i)< \mathcal{L}(Q^i_S,\theta_{S}^{\alpha})$ holds.
\end{lemma}
\begin{proof}
See proof in the supplementary.
\end{proof}
Observe that the expected loss $\mathcal{L}$ defined in Lemma~\ref{lemma1} is the supervised loss where one does have the label information. Our proposed target predictor $\theta_T$ achieves a supervised loss at least as good as the best individual source model. Importantly, the inequality is strict under a natural mild condition: The best individual source model $\beta$ (for the target $Q_T$) is strictly worse than some source model $i$ on the source distribution $Q_{S}^i$.
We also note the key differences between our algorithm and the predictor in Lemma \ref{lemma1}. In our algorithm's combination rule, we fine-tune the feature extractors of each source model unlike Lemma \ref{lemma1}. However each source has an individual weight which is agnostic to the source data, whereas Lemma \ref{lemma1} uses different weights per input instance.  Below we provide an intuitive justification for choosing this input-agnostic weighting strategy.




Since we do not know the source distributions (due to the unavailability of source data), let us consider the least informative of all the distributions i.e.~uniform distribution for sources by the \textit{Principle of Maximum Entropy} \cite{jaynes1957information}. This uniformity is assumed over the target support set $\mathcal{X}$. In what follows, we will consider the restrictions of the source distributions to the target support $\cal{X}$. Mathematically, our assumption is $Q^k_S(x)=c_k \mathcal{U}(x)$ when restricted to the support set $x  \in \mathcal{X}$, where $c_k$ is a scaling factor which captures the relative contribution of source $k$ and $\mathcal{U}(x)$ has value $1$. If we plug this value of the distribution in the combination rule in Lemma~\ref{lemma1}, we get $\theta_T(x)=\sum_{k=1}^n \frac{\lambda_k c_k}{\sum_{j=1}^n \lambda_j c_j} \theta_{S}^k(x)$ (see supplementary for more details).
This term consisting of $\lambda_k$ and $c_k$
essentially becomes the weighting term  $\alpha_k$ in our algorithm. 
We put this value of $\theta_T$ to solve the optimization~\eqref{opt:main_opt} jointly with respect to this $\alpha_k$ and $\phi_S^k$. 
Thus, our optimization will return us a favorable combination of source hypotheses, satisfying the bounds in Lemma \ref{lemma1}, under the uniformity assumption of source distributions.


\section{Experiments} \label{experiments}
\begin{table*}[t]
\centering
\resizebox{\textwidth}{!}{
\begin{tabular}{@{}llllllll@{}}
\toprule[1.2pt]
\textsc{Source}                       & \textsc{Method}        & \begin{tabular}[c]{@{}l@{}}\textsc{mt,up,sv,sy}\\ $\rightarrow$ \textsc{mm}\end{tabular} & \begin{tabular}[c]{@{}l@{}}\textsc{mm,up,sv,sy}\\ $\rightarrow$ \textsc{mt}\end{tabular} & \begin{tabular}[c]{@{}l@{}}\textsc{mm,mt,sv,sy}\\ $\rightarrow$ \textsc{up}\end{tabular} & \begin{tabular}[c]{@{}l@{}}\textsc{mm,mt,up,sy}\\ $\rightarrow$ \textsc{sv}\end{tabular} & \begin{tabular}[c]{@{}l@{}}\textsc{mm,mt,up,sv}\\ $\rightarrow$ \textsc{sy}\end{tabular} & 
\textsc{Avg.} \\ \midrule[1.1pt]

\multirow{6}{*}{Multiple(w)} & DAN\cite{long2015learning}   &63.7 &96.31 &94.2 &62.5 &85.4 &80.4 \\
                        &DANN\cite{ganin2015unsupervised}   &71.3 &97.6 &92.3 &63.5 &85.3 &82.0 \\ 
                        &MCD\cite{saito2018maximum} &72.5 &96.21 &95.3 &78.9 &87.5 &86.1\\
                        &CORAL\cite{sun2015return} &62.5 &97.2 &93.4 &64.4 &82.7 &80.1\\
                        &ADDA\cite{tzeng2017adversarial} &71.6 &97.9 &92.8 &75.5 &86.5 &84.8\\
                        &$\text{M}^3$SDA-$\beta$\cite{peng2019moment}
                        &72.8 &98.4 &96.1 &81.3 &89.6 &87.6 \\
\midrule

\multirow{4}{*}{Single(w/o)} & Source-best   &60.7 &98.2 &74.5 &89.5 &89.4 &82.5 \\
                        & Source-worst   &21.3 &64 &29.3 &7.4 &25.7 &29.5 \\
                        & SHOT\cite{liang2020we}-best  &\textbf{94.0} &98.7 &\textbf{97.9} &\textbf{83.5} &\textbf{97.5} &\textbf{94.3} \\ 
                        & SHOT\cite{liang2020we}-worst &44.5 &97.2 &96.2 &29.5 &32.5 &60.0\\
\midrule
\multirow{2}{*}{Multiple(w/o)} & SHOT\cite{liang2020we}-Ens &90.4 &98.9 &97.7 &58.3 &83.9 &85.8 \\
                          & DECISION(Ours)     &93.0 &\textbf{99.2} &97.8 &82.6 &\textbf{97.5} &94.0 \\
\bottomrule\\

\end{tabular}
}

\caption{\textbf{Results on digit recognition.} MT, MM, UP, SV, SY are abbreviations of \textit{MNIST, MNIST-M, USPS, SVHN} and \textit{Synthetic Digits} respectively. Multiple and Single denotes the methods which uses multiple and single sources respectively for domain adaptation, while (w) and (w/o) are abbreviations of \textit{with source data} and \textit{without source data} respectively. \textit{Source} is the accuracy with the unadapted models, whereas \textit{-best} and \textit{-worst} refer to the best and worst sources.
}
\label{tab:digit}
\end{table*}
\begin{table}[t]
\resizebox{0.48\textwidth}{!}{
\begin{tabular}{@{}llllll@{}}
\toprule[1.2pt]
\textsc{Source}                       & \textsc{Method}        & \begin{tabular}[c]{@{}l@{}}\textsc{A,D}\\ $\rightarrow$ \textsc{W}\end{tabular} & \begin{tabular}[c]{@{}l@{}}\textsc{A,W}\\ $\rightarrow$ \textsc{D}\end{tabular} & \begin{tabular}[c]{@{}l@{}}\textsc{D,W}\\ $\rightarrow$ \textsc{A}\end{tabular} & 
\textsc{Avg.} \\ \midrule[1.1pt]
\multirow{4}{*}{Single} & Source-best   &96.3 &98.4 &62.5 &85.7 \\
                        & Source-worst  &75.6 &80.9 &62.0 &72.8 \\  
                        & SHOT \cite{liang2020we}-best  &98.2 &99.6 &75.1 &90.9 \\ 
                        & SHOT \cite{liang2020we}-worst &90.6 &94.2 &72.9 &85.9 \\
\midrule
\multirow{2}{*}{Multiple} & SHOT \cite{liang2020we}-Ens &94.9 &97.8 &75.0 &89.3 \\
                          & DECISION(Ours)     &\textbf{98.4} &\textbf{99.6} &\textbf{75.4} &\textbf{91.1} \\
\bottomrule \\
\end{tabular}
}
\caption{\textbf{Results on Office}: A,D and W are abbreviations of \textit{Amazon}, \textit{DSLR} and \textit{Webcam}. For single source methods, Source-best and Source-worst denote the best and worst unadapted source models, whereas SHOT-best, SHOT-worst are the best and worst accuracies of adapted source models.  
}
\label{tab:office}
\end{table}
\noindent\textbf{Datasets.} To test the effectiveness of our algorithm, we experiment on various visual benchmarks described as follows. \\
    $\bullet$ \textit{Office} \cite{hoffman2018cycada}: In this benchmark DA dataset there are three domains under the office environment namely Amazon (A), DSLR (D) and Webcam (W) with a total of 31 object classes in each domain.\\
    $\bullet$ \textit{Office-Caltech} \cite{gong2012geodesic}: This is an extension of the Office dataset, with Caltech-256 (C) added on top of the 3 existing domains by extracting 10 classes common to all domains.\\
    $\bullet$ \textit{Office-Home} \cite{venkateswara2017deep}: Office-Home consists of four domains, namely, Art(Ar), Clipart(Cl), Product(Pr) and Real-world(Re). Each of these domains contain 65 object classes.\\
    $\bullet$ \textit{Digits}: The Digits dataset is a benchmark for DA in digit recognition. Following \cite{peng2019moment}, we utilize five subsets, namely MNIST (MT), USPS (UP), SVHN (SV), MNIST-M (MM) and Synthetic Digits (SY) for our experiments.

In all of our experiments, we take turns and fix one of the domains as the target and the rest as the source domains. The source data is discarded after training the source models.

\noindent \textbf{Baseline Methods.} We compare our method against a wide array of baselines. Similar to our setting, SHOT \cite{liang2020we} attempts unsupervised adaptation without source data. However, it adapts a single source at a time. We compare against a multi-source extension of SHOT via ensembling - we pass the target data through each of the adapted source model and take an average of the soft prediction to obtain the test label. In our comparisons, we name this method SHOT-ens. We also compare against single source baselines, namely SHOT-best and SHOT-worst, which refer to the best adapted source model and the worst one respectively, learned using SHOT. Additionally, we run comparisons against traditional multi-source adaptation methods $\text{M}^3$SDA-$\beta$\cite{peng2019moment}, DAN \cite{long2015learning}, DANN \cite{ganin2015unsupervised}, MCD \cite{saito2018maximum}, CORAL \cite{sun2015return}, ADDA \cite{tzeng2017adversarial}, DCTN\cite{xu2018deep}. All these methods, except for SHOT, have access to source data during adaptation.
\begin{table*}[t]
\small
\centering
\begin{tabular}{@{}lllllll@{}}
\toprule[1.2pt]
\textsc{Source}                       & \textsc{Method}        & \begin{tabular}[c]{@{}l@{}}\textsc{Ar,Cl,Pr}\\ $\rightarrow$ \textsc{Rw}\end{tabular} & \begin{tabular}[c]{@{}l@{}}\textsc{Ar,Cl,Rw}\\ $\rightarrow$ \textsc{Pr}\end{tabular} & \begin{tabular}[c]{@{}l@{}}\textsc{Ar,Pr,Rw}\\ $\rightarrow$ \textsc{Cl}\end{tabular} & 
\begin{tabular}[c]{@{}l@{}}\textsc{Cl,Pr,Rw}\\ $\rightarrow$ \textsc{Ar}\end{tabular} & 
\textsc{Avg.} \\ \midrule[1.1pt]
\multirow{4}{*}{Single(w/o)} & Source-best   &74.1 &78.3 &46.2 &65.8 &66.1\\
                        & Source-worst   &64.8 &62.8 &40.9 &53.3 &55.5\\
                        & SHOT\cite{liang2020we}-best &81.3 &83.4 &57.2 &72.1 &73.5\\ 
                        & SHOT\cite{liang2020we}-worst &80.8 &77.9 &53.8 &66.6 &69.8\\ 
\midrule
\multirow{2}{*}{Multiple(w/o)} & SHOT\cite{liang2020we}-Ens &82.9 &82.8 &59.3 &72.2 &74.3\\
                          & DECISION(Ours)    &\textbf{83.6} &\textbf{84.4} &\textbf{59.4} &\textbf{74.5} &\textbf{75.5}\\
\bottomrule \\
\end{tabular}
\caption{\textbf{Results on Office-Home.}: AR,CL,RW and PR are abbreviations of \textit{Art}, \textit{Clipart},\textit{Real-world} 
and \textit{Product}. We see that our method outperforms all the baselines including the best source accuracy as well as ensemble method. The abbreviations under the column SOURCE and METHOD are same as described in Table~\ref{tab:digit}.}
\label{tab:office-home}
\end{table*}
\begin{table*}[t]
\small
\centering
\begin{tabular}{@{}lllllll@{}}
\toprule[1.2pt]
\textsc{Source}                       & \textsc{Method}        & \begin{tabular}[c]{@{}l@{}}\textsc{A,C,D}\\ $\rightarrow$ \textsc{W}\end{tabular} & \begin{tabular}[c]{@{}l@{}}\textsc{A,C,W}\\ $\rightarrow$ \textsc{D}\end{tabular} & \begin{tabular}[c]{@{}l@{}}\textsc{C,D,W}\\ $\rightarrow$ \textsc{A}\end{tabular} & \begin{tabular}[c]{@{}l@{}}\textsc{A,D,W}\\ $\rightarrow$ \textsc{C}\end{tabular} 
& \textsc{Avg.} \\ \midrule[1.1pt]
\multirow{5}{*}{Multiple(w)} & ResNet-101\cite{he2016deep}   &99.1 &98.2 &88.7 &85.4 &92.9  \\
                             & DAN\cite{long2015learning}   &99.5 &99.1 &91.6 &89.2 &94.8  \\
                             & DCTN\cite{xu2018deep}   &99.4 &99.0 &92.7 &90.2 &95.3  \\
                             & MCD\cite{saito2018maximum}   &99.5 &99.1 &92.1 &91.5 &95.6  \\
                             &$\text{M}^3$SDA-$\beta$\cite{peng2019moment}   &99.5 &99.2 &94.5 &99.2 &96.4  \\
                             
\midrule
\multirow{4}{*}{Single(w/o)} & Source-best   &98.9 &99.3 &94.8 &86.5 &94.9\\
                        & Source-worst  &86.7 &89.8 &89.6 &83.2 &87.4\\  
                        & SHOT-best  &99.6 &100 &95.8 &95.5 &97.7\\ 
                        & SHOT-worst &97.3 &96.2 &95.7 &93.9 &95.8\\
\midrule
\multirow{2}{*}{Multiple(w/o)} & SHOT-Ens &99.6 &96.8 &95.7 &95.8 &97.0\\
                          & DECISION(Ours)     &\textbf{99.6} &\textbf{100} &\textbf{95.9} &\textbf{95.9}  &\textbf{98.0}\\
\bottomrule\\

\end{tabular}
\caption{\textbf{Results on Office-Caltech Dataset}:A,D,C and W are abbreviations of \textit{Amazon}, \textit{DSLR}, \textit{Caltech-256} and \textit{Webcam}. Our method consistently outperform all the baselines across all the domains as target.The abbreviations under the column SOURCE and METHOD are same as described in Table~\ref{tab:digit}.}

\label{tab:office-cal}
\end{table*}

\subsection{Implementation details} 
\noindent \textbf{Network architecture.} For the object recognition tasks, we use a pre-trained ResNet-50 \cite{he2016deep} as the feature extractor backbone, similar to \cite{peng2019moment,xu2019larger}. Following \cite{liang2020we, ganin2015unsupervised}, we replace the penultimate fully-connected layer with a bottleneck layer and a task specific classifier layer. Batch normalization \cite{ioffe2015batch} is utilized after the bottleneck layer, along with weight normalization \cite{salimans2016weight} in the final layer. For the digit recognition task, we use a variant of the LeNet \cite{lecun1998gradient} similar to \cite{liang2020we}. \\
\noindent \textbf{Source model training.} Following \cite{liang2020we}, we train the source models using smooth labels, instead of the usual one-hot encoded labels. This increases the robustness of the model and helps in the adaptation process by encouraging features to lie in tight,
evenly separated clusters \cite{muller2019does}. The maximum number of epochs for Digits, Office, Office-Home and Office-Caltech is set to 30, 100, 50 and 100, respectively. Additionally, for our experiments on digit recognition, we resize images from each domain to 32$\times$32 and convert the gray-scale images to RGB.\\
\noindent \textbf{Hyper-parameters.} The entire framework is trained in an end-to-end fashion via back-propagation. Specifically, we utilize stochastic gradient descent with momentum value $0.9$ and weight decay equalling $10^{-3}$. The learning rate is set at $10^{-2}$ for the bottleneck and classifier layers, while the backbone is trained at a rate of $10^{-3}$. In addition, we use the learning rate scheduling strategy from \cite{ganin2015unsupervised}, where the initial rate is exponentially decayed as learning progresses. The batch size is set to $32$. We use $\lambda=0.3$ for all the object recognition tasks and $\lambda=0.1$ for the digits benchmark. For adaptation, maximum number of epochs is set to 15, with the pseudo-labels updated at the start of every epoch. We use PyTorch \cite{paszke2019pytorch} for all our experiments.

\subsection{Digit recognition}
The results on digit recognition are shown in Table \ref{tab:digit}. The digit benchmark is characterised by the presence of very poor sources in some scenarios, notably when treating MNIST-M, SVHN or Synthetic Digits as the target domain. For example, on SVHN as the target, the best and worst source models adapted using SHOT \cite{liang2020we} exhibit a performance gap of more than $50\%$. Combining these models via uniform ensembling results in a predictor which greatly underperforms the best adapted source. In contrast, our method restricts this severe negative transfer via a joint adaptation over the models and the ensembling weights, and outperforms the baseline by $\mathbf{24.3}\%$. The corresponding increase in performance when using Synthetic Digits and MNISTM as the target are $\mathbf{13.5}\%$ and $\mathbf{2.6}\%$ respectively. Overall, we obtain an average increase of $\mathbf{8.2}\%$ across all the digit adaptation tasks over SHOT-Ens. In spite of such disparities among the sources, our framework also achieves performance at par with the best adapted source and actually outperforms the latter on the MNIST transfer task. We also outperform the traditional multi-source adaptation methods, which use source data, on all the tasks by an average of $\mathbf{6.4}\%$.

\subsection{Object recognition}
\noindent \textbf{Office.} The results for the 3 adaptation tasks on the Office dataset are shown in Table \ref{tab:office}. We achieve performance at par with the best adapted source models on all the tasks and obtain an average increase of $\mathbf{5.2}\%$ over SHOT-Ens. In the task of adapting to the Webcam (W) domain, negative transfer from the Amazon (A) model brings the ensemble performance down - our model is able to prevent this, and not only outperforms the ensemble by $\mathbf{3.5}\%$ but also achieves higher performance than the best adapted source. 
\\
\noindent \textbf{Office-Home.} On the Office-Home dataset, we outperform all baselines as shown in Table \ref{tab:office-home}. Across all tasks, our method achieves a mean increase in accuracy of $\mathbf{2}\%$ over the respective best adapted source models. This can be attributed to the relatively small performance gap between the best and worst adapted sources in comparison to other datasets. This suggests that, as the performance gap between the best and worst performing sources gets smaller, or outlier sources are removed, our method can generalize even better to the target.
\\
\noindent \textbf{Office-Caltech.} The results follow a similar trend on the Office-Caltech dataset, as shown in Table \ref{tab:office-cal}. With a mean accuracy of $\mathbf{98}\%$ across all tasks, we outperform all baselines.

\subsection{Ablation study}
\noindent \textbf{Contribution of each loss.} Our framework is trained using a combination of three distinct losses: $\mathcal{L}_{\text{div}}$, $\mathcal{L}_{\text{ent}}$ and $\mathcal{L}_{\text{pl}}$. We study the contribution of each component of our framework to the adaptation task in Table \ref{tab:ablation}. First, we remove both the diversity loss and the pseudo-labeling, and train using only $\mathcal{L}_{\text{ent}}$. Next, we add in $\mathcal{L}_{\text{div}}$ and perform weighted information maximization. Finally, we also compare the results of solely using $\mathcal{L}_{\text{pl}}$.\\

\begin{table}[h]
\small
\centering
\begin{tabular}{@{}lllll@{}}
\toprule[1.2pt]
\textsc{Method}        & \begin{tabular}[c]{@{}l@{}}\textsc{A,D}\\ $\rightarrow$ \textsc{W}\end{tabular} & \begin{tabular}[c]{@{}l@{}}\textsc{A,W}\\ $\rightarrow$ \textsc{D}\end{tabular} & \begin{tabular}[c]{@{}l@{}}\textsc{D,W}\\ $\rightarrow$ \textsc{A}\end{tabular} & 
\textsc{Avg.} \\ \midrule[1.1pt]
$\mathcal{L}_{pl}$               &97.6 &98.5 &75.3 &90.5  \\
$-\mathcal{L}_{ent}$  &96.6 &99.0 &68.5 &88.0 \\  
$-\mathcal{L}_{ent} + \mathcal{L}_{div}$  &95.9 &99.0 &71.6 & 88.9\\ 
$-\mathcal{L}_{ent} + \mathcal{L}_{div} + \lambda\mathcal{L}_{pl}$  &\textbf{98.4} &\textbf{99.6} &\textbf{75.4} &\textbf{91.1} \\

\bottomrule \\
\end{tabular}

\caption{\textbf{Loss-wise ablation.} Contribution of each component in adaptation on the Office dataset.}
\label{tab:ablation}
\end{table}

\noindent \textbf{Analysis on the learned weights.}
Our framework jointly adapts the the source models and learns the weights on each such source. To understand the impact of the weights, we propose to freeze the feature extractors and optimize solely over the weights $\{\alpha_j\}_{i=1}^n$. Naturally, this setup yields better performance compared to trivially assigning equal weights to all source models, as shown in Table \ref{tab:weights}. More interestingly, the learned weights correctly indicate which source model performs better on the target and could serve as a proxy indicator in a model selection framework. See Figure \ref{fig:weights}. \\
\begin{table}[]
\scriptsize
\centering
\scalebox{0.9}{
\begin{tabular}{@{}llllll@{}}
\toprule[1.2pt]
\textsc{Method}
& \begin{tabular}[c]{@{}l@{}}\textsc{Ar,Cl,Pr}\\ $\rightarrow$ \textsc{Rw}\end{tabular} 
& \begin{tabular}[c]{@{}l@{}}\textsc{Ar,Cl,Rw}\\ $\rightarrow$ \textsc{Pr}\end{tabular} 
& \begin{tabular}[c]{@{}l@{}}\textsc{Ar,Pr,Rw}\\ $\rightarrow$ \textsc{Cl}\end{tabular} 
& \begin{tabular}[c]{@{}l@{}}\textsc{Cl,Pr,Rw}\\ $\rightarrow$ \textsc{Ar}\end{tabular} & \textsc{Avg.}  \\ \midrule
Source-Ens       
& 67.6 
& 51.4  
& 77.7
& 80.1     
& 69.2 \\
DECISION-weights 
& \textbf{68.8}  
& \textbf{52.3}
& \textbf{79.2}  
& \textbf{80.4}   
& \textbf{70.2} \\
\bottomrule \\
\end{tabular}
}
\caption{\textbf{Performance on freezing backbone network on Office-Home.} DECISION-weight is optimized solely over the source weights and consistently performs better than uniform weighting.}
\label{tab:weights}
\end{table}
\begin{figure}
    \centering
    \includegraphics[width=0.48\textwidth]{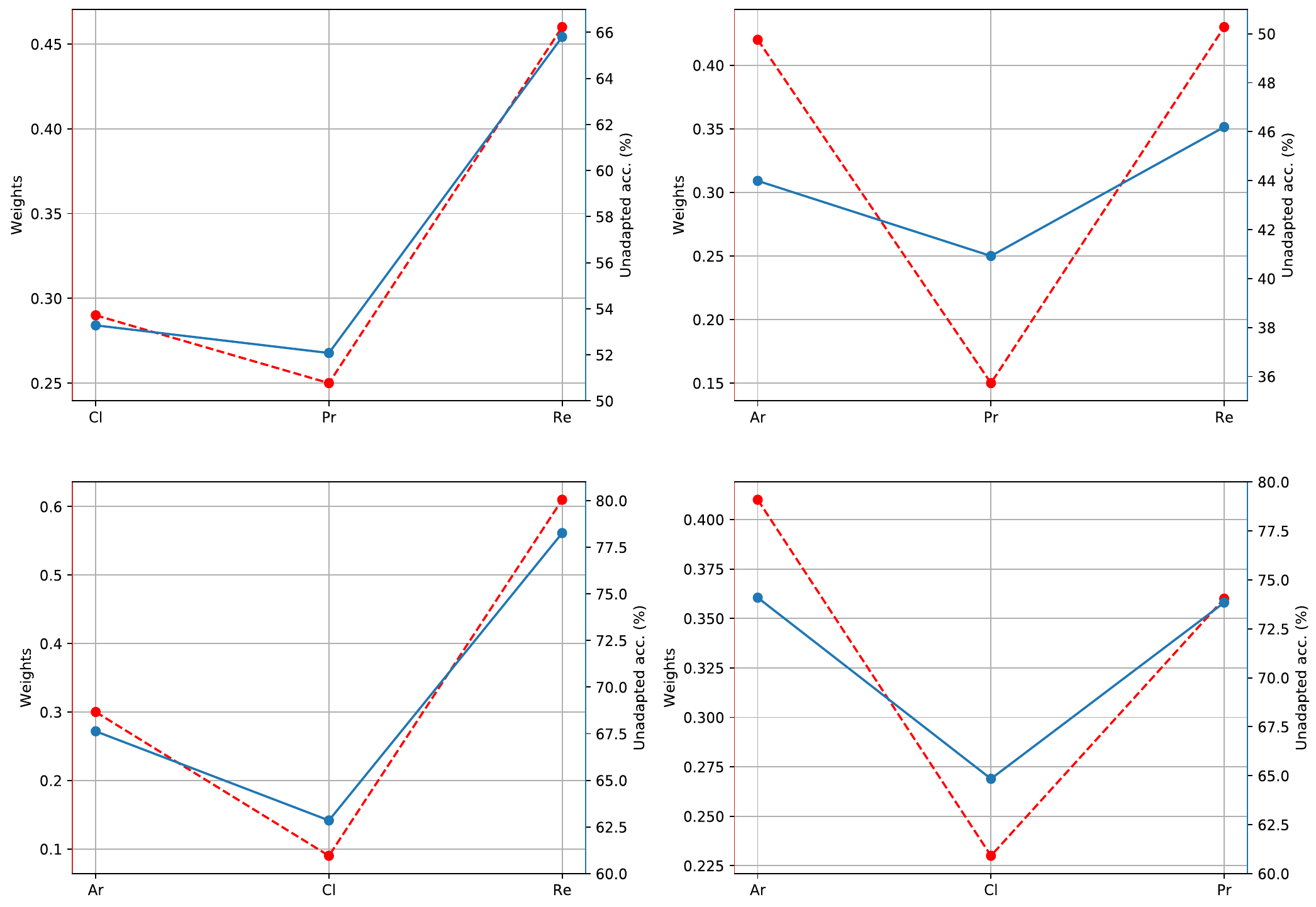}
    \caption{\textbf{Weights as model selection proxy.} The weights learnt by our framework on Office-Home correlates positively with the unadapted source model performance. (Left axis corresponds to the red plot and right to the blue plot, best viewed in color.)}
    \label{fig:weights}
    \vskip -0.1in
\end{figure}

\noindent \textbf{Distillation into a single model.} Since we are dealing with multiple source models, inference time is of the order $\mathcal{O}(m)$, where $m$ is the number of source models. If $m$ is large, this can lead to inference being quite time consuming. To ameliorate this overhead, we follow a knowledge distillation \cite{hinton2015distilling} strategy to obtain a single target model. Teacher supervision is obtained by linearly combining the adapted models via the learned weights. These annotations are subsequently used to train the single student model via vanilla cross-entropy loss. Results obtained using this strategy are presented in the supplementary.

\section{Conclusions and Future Work}
We developed a new UDA algorithm that can learn from and optimally combine multiple source models without requiring source data. We provide theoretical intuitions for our algorithm and verify its effectiveness in a variety of domain adaptation benchmarks. There are multiple exciting directions to pursue including: First, we suspect that our algorithm's performance can be further boosted by incorporating data augmentation techniques during training. Second, when there are too many source models to utilize, it would be interesting to study whether we can automatically select an optimal subset of the source models without requiring source data in an unsupervised fashion.

\noindent\textbf{Acknowledgements.}
This work was partially supported by the ONR grant N00014-19-1-2264 and the NSF grants 2008020, 2046816.

{\small
\bibliographystyle{ieee_fullname}
\bibliography{main}
}
\onecolumn
\newpage




\begin{center}
\textbf{\Large{Unsupervised Multi-source Domain Adaptation Without Access to Source Data \\
\vspace{1cm}
(Supplementary Material)}}
\end{center}

\newpage
\section{Proof of Lemma 1}
\begin{lemma}
Assume that the loss $L(\theta(x),y)$ is convex 
in its first argument and that there exists a $\lambda \in \mathbb{R}^n$ where $\lambda \geq 0$ and $\lambda^\top \mathbbm{1}=1$, such that the target distribution is exactly equal to the mixture of source distributions, i.e $Q_T=\sum_{i=1}^n \lambda_i Q^i_S$. Set the target predictor as the following convex combination of the optimal source predictors
$$\theta_T(x)=\sum_{k=1}^n \frac{\lambda_k Q^k_S(x)}{\sum_{j=1}^n \lambda_j Q^j_S(x))} \theta_{S}^k(x).$$ 
Recall the pseudo-labeling loss (\textcolor{red}{$10$}). Then, for this target predictor, over the target distribution, the unsupervised loss induced by the pseudo-labels and the supervised loss are both less than or equal to the loss induced by the best source predictor. In particular, 
\[
 \mathcal{L}(Q_T,\theta_T) \leq \underset{1\leq j\leq n}{\min}\ \mathcal{L}(Q_T,\theta_{S}^j).
\]
\end{lemma}


\begin{proof}
We can see that the left hand-side of the inequality can be upper-bounded by some loss as follows,

\begin{equation}
\begin{split}
\mathcal{L}(Q_T,\theta_T)= \int_x Q_T(x) L(\theta_T(x),y)=& \int_x Q_T(x) L\Bigg(\sum_{i=1}^n \frac{\lambda_i Q^i_S(x)}{\sum_{j=1}^n \lambda_j Q^j_S(x))} \theta_{S}^i(x),y\Bigg)dx\\
& \leq \int_x Q_T(x) \sum_{i=1}^n \frac{\lambda_i Q^i_S(x)}{\sum_{j=1}^n \lambda_j Q^j_S(x))} L(\theta_{S}^i(x),y) dx~~\ \text{(from Jensen's inequality)} \\
& = \int_x Q_T(x) \sum_{i=1}^n\frac{\lambda_i Q^i_S(x)}{Q_T(x)} L(\theta_{S}^i(x),y)dx ~~\ \text{(from distribution assumption)} \\
& = \sum_{i=1}^n \lambda_i \int_x Q^i_S(x)L(\theta_{S}^i(x),y) dx~~\ \text{(changing the order of summation)} \\
& = \sum_i \lambda_i \mathcal{L}(Q^i_S(x),\theta_{S}^i)
\end{split}
\end{equation}

Now for the R.H.S.~we can write this loss as follows,
\begin{equation}
\begin{split}
    \mathcal{L}(Q_T,\theta_{S}^j) &= \int_{x}  Q_T(x) L(\theta_{S}^j(x),y)dx \\
    &= \int_{x} \sum_{i=1}^n \lambda_i Q^i_S(x)  L(\theta_{S}^j(x),y)dx \\
    &= \sum_{i=1}^n \lambda_i \int_{x} Q^i_S L(\theta_{S}^j(x),y)dx \\
    &= \sum_{i=1}^n \lambda_i  \mathcal{L}(Q^i_S(x),\theta_{S}^j)
\end{split}
\end{equation}
Now recall from main paper that,
\[
\theta^k_S = \arg\min_{\theta} {\cal{L}}(Q^k_S, \theta)\quad \text{for}\quad 1\leq k\leq n.
\].

This means $\theta_{S}^i$ is the best predictor for the source $i$, which has distribution $Q^i_S$. Thus we find that $\mathcal{L}(Q^i_S,\theta_{S}^i) \leq \mathcal{L}(Q^i_S,\theta_{S}^j) \ \forall j$, which implies $\sum_i \lambda_i \mathcal{L}(Q^i_S,\theta_{S}^i) \leq \sum_i  \lambda_i \mathcal{L}(Q^i_S,\theta_{S}^j)$. This further implies that $\mathcal{L}(Q_T,\theta_T) \leq  \mathcal{L}(Q_T,\theta_{S}^j) \ \forall j$, which in turn concludes the proof $\mathcal{L}(Q_T,\theta_T) \leq \underset{1\leq j\leq n}{\text{min}}\ \mathcal{L}(Q_T,\theta_{S}^j)$. Finally, suppose the entries of $\lambda$ are strictly positive and let $\beta=\arg\min_j \mathcal{L}(Q_T,\theta^j_S)$. Observe that, if there is a source $i$ such that the strict inequality $\mathcal{L}(Q^i_S,\theta_{S}^i)< \mathcal{L}(Q^i_S,\theta_{S}^{\beta})$ holds, then the main claim of the lemma also becomes strict as we find 
\[
\mathcal{L}(Q_T,\theta_T)\leq \sum_i \lambda_i \mathcal{L}(Q^i_S,\theta_{S}^i) < \sum_i  \lambda_i \mathcal{L}(Q^i_S,\theta_{S}^{\beta})\leq \min_j \mathcal{L}(Q_T,\theta^j_S).
\]
Verbally, this strict inequality has a natural meaning that the model $j$ is strictly worse than model $i$ for the source data $i$.
\end{proof}

\section{Detailed steps of combination rule under source distribution uniformity assumption}
See the discussion after \textbf{Lemma 1} in the main paper for reference. 
\begin{equation}
\label{uniform}
\begin{split}
    \theta_T(x)&=\sum_{k=1}^n \frac{\lambda_k Q^k_S(x)}{\sum_{j=1}^n \lambda_j Q^j_S(x)} \theta_{S}^k(x)\\
    &=\sum_{k=1}^n \frac{\lambda_k c_k \mathcal{U}(x)}{\sum_{j=1}^n \lambda_j c_j \mathcal{U}(x)} \theta_{S}^k(x) \\
    &= \sum_{k=1}^n \frac{\lambda_k c_k}{\sum_{j=1}^n \lambda_j c_j} \theta_{S}^k(x) \\
\end{split}
\end{equation}

\section{Additional Experiments}

\begin{figure}[ht]
    \centering
    \includegraphics[width=\textwidth]{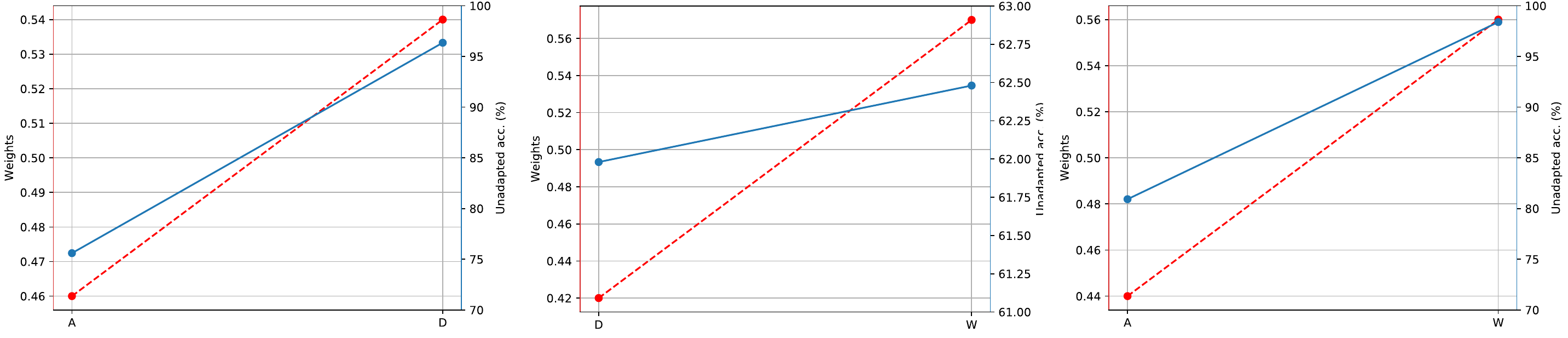}
    \caption{\textbf{Weights as model selection proxy.} The weights learnt by our framework on Office-31 correlates positively with the unadapted source model performance. (Left axis corresponds to the red plot and right to the blue plot, best viewed in color.) }
    \label{fig:weights_office31}
\end{figure}
From Figure~\ref{fig:weights_office31}, we can clearly see that for the model which gives higher accuracy for the unadapted scenario, it is automatically given higher weightage by our algorithm. As a result, we can easily infer about the quality of the source domain, in relation to the target, from the weights learnt by our framework. \\

\noindent \textbf{Effect of weight on pseudo-labeling.} We investigate the effect of the weight $\lambda$ on $\mathcal{L}_{\text{pl}}$. We perform experiments on the Office dataset by varying the value of $\lambda$ and plot the results in Figure \ref{fig:hyperparam}. As shown in the plot, the proposed method performs best at $\lambda=0.3$

\begin{figure}[H]
    \centering
    \includegraphics[width=0.46\textwidth]{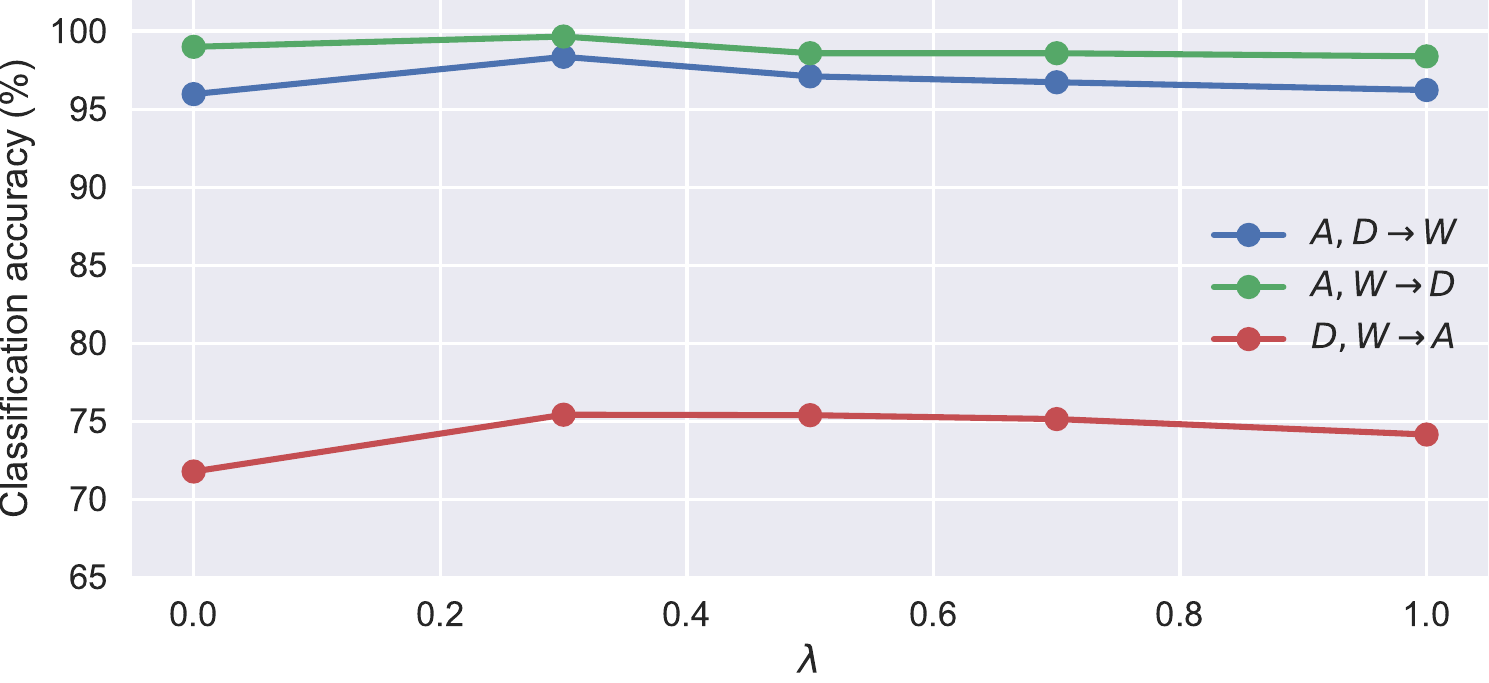}
    \caption{\textbf{Effect of $\lambda$.} The variations in classification as the weight on $\mathcal{L}_{\text{pl}}$ is varied. (Best viewed in color)}
    \label{fig:hyperparam}
\end{figure}

\noindent\textbf{Effect of outlier source models.} Our method is clearly robust to outlier source models. In Table \textcolor{red}{2} of the main paper, when \textit{MNIST-M} is the target, transferring from only \textit{USPS}, leads to an extremely poor performance of \textbf{21.3$\%$} - here, USPS is a strong outlier. Despite the presence of such a poor source, our framework is mostly able to correctly negate the transfer from USPS, achieving a performance of \textbf{93$\%$}, close to the best source performance of \textbf{94$\%$}. On removing \textit{USPS} as a source, DECISION outperforms the best source by achieving an accuracy of \textbf{94.5$\%$}. In the future, we plan to actively use the weights to simultaneously remove poor sources while adaptation in order to boost the performance.\\

\noindent\textbf{DomainNet [\textcolor{green}{32}]:}
 This is a relatively new and large dataset where there are six domains under the common object categories, namely quickdraw (Q), clipart (C), painting (P), infograph (I), sketch (S) and real (R) with a total of 345 object classes in each domain. Experimental results on this dataset are shown in Table~\ref{tab:domain-net}. Our method consistently outperforms the best adapted source baselines (SHOT-best) except for \textit{infograph} as a target. However the average performance over all the domains as target is slightly less than the SHOT-Ens. Note that for \textit{quickdraw}  and \textit{clipart} as target, our method outperforms all the state of the art methods including source free and with source data single and multi source state-of-the-art DA methods. \\
\begin{table*}[t]
\small
\centering
\begin{tabular}{@{}lllllllll@{}}
\toprule[1.2pt]
\textsc{Source}                       & \textsc{Method}        & \begin{tabular}[c]{@{}l@{}}\textsc{C,P,I,S,R}\\ $\rightarrow$ \textsc{Q}\end{tabular} & \begin{tabular}[c]{@{}l@{}}\textsc{Q,P,I,S,R}\\ $\rightarrow$ \textsc{C}\end{tabular} & \begin{tabular}[c]{@{}l@{}}\textsc{Q,C,I,S,R}\\ $\rightarrow$ \textsc{P}\end{tabular} & \begin{tabular}[c]{@{}l@{}}\textsc{Q,C,P,S,R}\\ $\rightarrow$ \textsc{I}\end{tabular} &
\begin{tabular}[c]{@{}l@{}}\textsc{Q,C,P,I,R}\\ $\rightarrow$ \textsc{S}\end{tabular} &
\begin{tabular}[c]{@{}l@{}}\textsc{Q,C,P,I,S}\\ $\rightarrow$ \textsc{R}\end{tabular}

& \textsc{Avg.} \\ \midrule[1.1pt]
\multirow{5}{*}{Multiple(w)} 
                             & DAN[\textcolor{green}{25}]   &16.2 &39.1 &33.3 &11.4 &29.7 &42.1 &28.6 \\
                             & DCTN[\textcolor{green}{46}] &7.2 &48.6 &48.8 &23.4 &47.3 &53.5 &38.1 \\
                             & MCD[\textcolor{green}{37}]  &7.6 &54.3 &45.7 &22.1 &43.5 &58.4 &38.6 \\
                             &$\text{M}^3$SDA-$\beta$[\textcolor{green}{32}]  &6.3 &58.6 &52.3 &26 &49.5  &62.7 &42.5\\
                             
\midrule
\multirow{4}{*}{Single(w/o)} & Source-best   &11.9 &49.9 &47.5 &20 &41.1 &57.7 &38\\
                        & Source-worst  &2.3 &12.2 &2.2 &1.1 &8.7 &4.8 &5.2\\  
                        & SHOT[\textcolor{green}{22}]-best  &18.7 &58.3 &53 &22.7 &48.4 &65.9 &44.5\\ 
                        & SHOT[\textcolor{green}{22}]-worst &3.8 &14.8 &3.5 &1 &11.9 &6.6 &7\\
\midrule
\multirow{2}{*}{Multiple(w/o)} & SHOT[\textcolor{green}{22}]-Ens &15.3 &58.6 &\textbf{55.3} &\textbf{25.2} &\textbf{52.4} &\textbf{70.5} &\textbf{46.2}\\
                          & DECISION(Ours)     &\textbf{18.9} &\textbf{61.5} &54.6 &21.6  &51 &67.5 &45.9\\
\bottomrule
\end{tabular}
\vspace{0.5cm}
\caption{\textbf{Results on DomainNet}:Q,C,P,I,S and R are abbreviations of \textit{quickdraw}, \textit{clipart}, \textit{painting}, \textit{infograph}, \textit{sketch} and \textit{real}.}

\label{tab:domain-net}
\end{table*}

\noindent \textbf{Distillation.} Our results on using the distillation strategy outlined in Section 5.4 of the main paper are shown in Table \ref{tab:distill}. Despite the model compression, the performance remains consistent.

\begin{table}[H]
\centering
\begin{tabular}{@{}lllll|llll|lll@{}}
\toprule
\textsc{Method} & \multicolumn{4}{c}{\textsc{Office-home}} & \multicolumn{4}{c}{\textsc{Office-Caltech}} & \multicolumn{3}{c}{\textsc{Office}} \\ \midrule
 & \textsc{Rw}   & \textsc{Pr}   & \textsc{Cl}   & \textsc{Ar}   & \textsc{A}       & \textsc{C}      & \textsc{D}      & \textsc{W}      & \textsc{A}       & \textsc{D}       & \textsc{W}      \\ \cmidrule(l){2-12} 
DECISION (original)  & 83.6 & 84.4 & 59.4 & 74.5 & 95.9    & 95.9   & 100    & 99.6   & 75.4    & 99.6    & 98.4   \\
DECISION (distillation)  & 83.7 & 84.4 & 59.1 & 74.4 & 96.0    & 95.7   & 99.4   & 99.6   & 75.4    & 99.6    & 98.1   \\ \bottomrule
\end{tabular}
\vspace{0.5cm}
\caption{\textbf{Distillation results on object recognition tasks.} Performance remains consistent across all datasets despite distilling into a single target model.}
\label{tab:distill}
\end{table}


\end{document}